\documentclass[conference]{IEEEtran}
\IEEEoverridecommandlockouts
\usepackage{cite}
\usepackage{amsmath,amssymb,amsfonts}
\usepackage{graphicx}
\usepackage{textcomp}
\usepackage{amssymb,amsmath,amsthm}
\newtheorem{theorem}{Theorem}[section]
\newtheorem{proposition}{Proposition}
\newtheorem{definition}{Definition}
\newtheorem{lemma}{Lemma}
\newtheorem{problem}{Problem}
\usepackage{booktabs}
\usepackage{algorithm}
\usepackage{algpseudocode}
\usepackage{multirow}
\usepackage{xcolor}
\def\BibTeX{{\rm B\kern-.05em{\sc i\kern-.025em b}\kern-.08em
    T\kern-.1667em\lower.7ex\hbox{E}\kern-.125emX}}
\begin{document}

\title{Hybrid Safety Verification of Multi-Agent Systems using $\psi$-Weighted CBFs and PAC Guarantees\\
}

\author{Venkat Margapuri, Garik Kazanjian, Naren Kosaraju
     \thanks{The authors are with the Department of  Computing Sciences at Villanova University, PA, USA. Emails: \{vmargapu, gkazanji, nkosaraj\}@villanova.edu} \\
}

\maketitle

\begin{abstract}
This study proposes a hybrid safety verification framework for closed-loop multi-agent systems under bounded stochastic disturbances. The proposed approach augments control barrier functions with a novel $\psi$-weighted formulation that encodes directional control alignment between agents into the safety constraints. Deterministic admissibility is combined with empirical validation via Monte Carlo rollouts, and a PAC-style guarantee is derived based on margin-aware safety violations to provide a probabilistic safety certificate. The results from the experiments conducted under different bounded stochastic disturbances validate the feasibility of the proposed approach. 
\end{abstract}


\section{Introduction}
Safety within multi-agent systems is essential for real-world applications such as autonomous driving \cite{fisher2021towards, araujo2023testing} and robotic swarm deployments in agriculture \cite{abdelkader2021aerial, abhang2024swarm}, manufacturing \cite{bi2021safety, yang2024plug}, and search and rescue operations \cite{drew2021multi}, where agents must navigate safely through their environment. Safety in a stochastic multi-agent dynamical system requires that all agent trajectories remain within a predefined safe set under specified control inputs and time horizons. 
Traditional approaches include reachability-based formulations \cite{borquez2024safety, bansal2017hamilton}, where a Hamilton-Jacobi partial differential equation is solved to characterize the backward-reachable set that avoids unsafe regions. However, these methods are computationally expensive and scale poorly to high-dimensional or multi-agent systems. More recently, barrier certificates \cite{10.1007/978-3-540-24743-2_32, 4287147} and control barrier functions (CBFs) \cite{7782377, ames2019control} have emerged as tractable alternatives for certifying safety. By enforcing forward invariance of a safe set via control-affine constraints, CBFs offer real-time safety guarantees under deterministic assumptions. Yet, such guarantees may fail in the presence of noise or unmodeled disturbances. While stochastic CBF variants address this, they often rely on strong distributional assumptions or chance-constrained formulations.

To bridge the gap, this work introduces a hybrid safety verification framework that unifies $\psi$-weighted CBFs for forward invariance with finite-sample probably approximately correct (PAC)-style guarantees for margin-aware safety under bounded stochastic disturbances, where $\psi$ is a term inspired by quantum walk dynamics \cite{qiang2024quantum} to promote pairwise safety among different agents. Rather than assuming complete knowledge of the noise distributions, the proposed method combines deterministic admissibility with empirical validation via Monte Carlo rollouts under bounded stochasticity, yielding a distribution-free safety certificate with high-probability guarantees. The proposed method is feasible in multi-agent applications where uncertainty is prevalent and exact noise modeling is infeasible. 

\section{Related Work}
CBFs are a class of Lyapunov-like \cite{chen2021learning} functions that enable formal safety guarantees by ensuring the forward invariance of a predefined safe set. For control-affine systems, a CBF defines a constraint on the control input that, when enforced through a real-time quadratic program (QP), ensures the system state remains within a safe set in perpetuity or a time horizon. Multiple studies \cite{clark2021verification, alan2023control, choi2021robust, 10383473} have adopted the use of CBFs for safety certification in control-affine systems across various domains. While classical CBF formulations operate under deterministic assumptions, they have been extended to higher-degree \cite{breeden2021high, wang2021learning, xiao2021high}, stochastic \cite{wang2021safety, singletary2022safe}, and hybrid \cite{maghenem2021adaptive, yang2024safe} systems, demonstrating the feasibility of CBFs for different classes of systems. In the realm of stochastic multi-agent settings, multiple studies \cite{song2022safety, 4287147} have explored enforcing safety by considering worst-case bounds on uncertainties. Although worst-case bounds offer robust guarantees, they can be overly conservative, keeping the system far from the safe set boundary. An alternative paradigm comes from providing probabilistic safety guarantees on constraints, where the safety condition is said to hold with a probability exceeding a certain threshold. Probabilistic safety guarantees offer a balanced approach by quantifying risk and allow for better trade-offs between safety and efficiency under stochasticity.

PAC guarantees \cite{gonzales2025multi, alquier2021user} and Scenario optimization techniques \cite{doi:10.1137/07069821X, 10383473} enable data-driven generalization from sampled realizations and provide a feasible framework to derive probabilistic guarantees for control-affine systems. Recent works \cite{gonzales2025multi, majumdar2021pac, 10383473, akella2022barrier} have proposed probabilistic safety guarantees based on PAC and scenario optimization approaches, avoiding assumptions about the underlying uncertainty distribution. Building on all these prior works, this study integrates $\psi$-weighted CBF-based control with PAC-style guarantees, enabling formal, distribution-free probabilistic guarantees for closed-loop multi-agent systems without requiring worst-case conservatism or strong assumptions about noise distributions.

\section{Preliminaries}

$\mathbb{N}, \mathbb{R}, \mathbb{R}_{\geq 0}$ denote the set of natural numbers, real numbers, and non-negative real numbers, respectively. The Euclidean norm denotes the magnitude of a vector, defined for any vector $v$ as $\|v\|_2 := \sqrt{v^\top v}$. The supremum of a function $f : \mathcal{D} \to \mathbb{R}$ over a domain $\mathcal{D}$ is denoted by $\sup_{x \in \mathcal{D}} f(x)$ and defined as the least upper bound of the image of $f$ over $\mathcal{D}$: \(
\sup_{x \in \mathcal{D}} f(x) := \inf \left\{ b \in \mathbb{R} \mid f(x) \leq b \ \forall x \in \mathcal{D} \right\}
\). A continuous function $\alpha: \mathbb{R} \to \mathbb{R}$ is said to be an extended class-$\mathcal{K}$ function if it is strictly increasing and $\alpha(0) = 0$. The standard notation $\mathbb{P}[\cdot]$ is used for probability measures, $exp(\cdot)$ for exponential functions, $\mathbb{E}[\cdot]$ for the expectation (average value) of a random variable over its probability distribution, and $\mathbb{I}[\cdot]$ to denote the indicator function (1 if true, 0 otherwise).

\section{Closed-Loop Multi-Agent Systems with Safety Constraints}

This section introduces the design of the closed-loop multi-agent systems considered for safety verification by this study. A closed-loop multi-agent system with $N$ agents is defined as tuple $\mathcal{S} = (\mathcal{X}, \mathcal{U}, \mathcal{W}, \mathcal{F}, \mathcal{G})$ where $\mathcal{X} = \{x_i \in \mathbb{R}^n \}_{i=1}^N$ is the state space, $\mathcal{U} = \{u_i \in \mathbb{R}^m \}_{i=1}^N$ is the control input space, and $\mathcal{W} = \{w_i(t) \in \mathbb{R}^n\}_{i=1}^N$ is the stochastic disturbance space modeled as a bounded process noise satisfying $\|w_i(t)\|_2 \leq \bar{w}$ for some constant $\bar{w} > 0$. $\mathcal{F} = \{f_i: \mathbb{R}^n \to \mathbb{R}^n\}_{i=1}^N$ and $\mathcal{G} = \{g_i: \mathbb{R}^n \to \mathbb{R}^{n \times m}\}_{i=1}^N$ are locally Lipschitz continuous, defining the non-linear control-affine dynamics. Each agent in the system evolves according to the dynamics: 
\begin{equation}
\label{eq:dynamics}
\dot{x}_i = f(x_i) + g(x_i) \cdot u_i + w_i(t)
\end{equation}

The agents are required to satisfy pairwise safety constraints specified through a continuously differentiable function $h_{ij} : \mathbb{R}^{Nn} \to \mathbb{R}$ that encodes a minimum separation or collision avoidance criterion between agents $i$ and $j$. 

\begin{definition}[$\psi$-Weighted Pairwise Safety Function]
\label{pairwise_safety_function}
For agents $i$ and $j$, the pairwise safety function incorporating amplitude-modulated influence is defined as:
\begin{equation}
\label{eq:safety_function}
\tilde{h}_{ij}(x,u) \triangleq h_{ij}(x) + \psi \cdot \mathcal{A}_{ij}^\top (u_i - u_j),
\end{equation}
where $h_{ij}(x): \mathbb{R}^{Nn} \to \mathbb{R}$ encodes a minimum separation constraint, and $\mathcal{A}_{ij}$ is a quantum-inspired propagation vector given by:
\begin{equation}
\label{eq:Aij_reg}
\mathcal{A}_{ij}(x)
\;\triangleq\;
\frac{x_i - x_j}{\sqrt{\|x_i - x_j\|_2^2 + \varepsilon^2}}
\,\exp\!\big(-\|x_i - x_j\|_2^2\big).
\end{equation}
representing a directionally coherent yet distance-sensitive modulation of control coupling. The scalar $\psi \geq 0$ adjusts the overall amplitude of influence, and $\epsilon > 0$ is a tiny constant to avoid division by zero.
\end{definition}

This formulation draws inspiration from quantum walk dynamics \cite{qiang2024quantum}, where amplitudes attenuate over spatial separation while preserving directional coherence. In this context, $\mathcal{A}_{ij}$ acts analogously to a coherence-weighted interaction term, suppressing long-range interference and enhancing influence between nearby agents. The operational flow of $\mathcal{S}$ with key functional components is illustrated in Figure \ref{fig:closed_loop_states}.

\begin{figure}[ht]
  \centering
  \includegraphics[width=0.7\columnwidth]{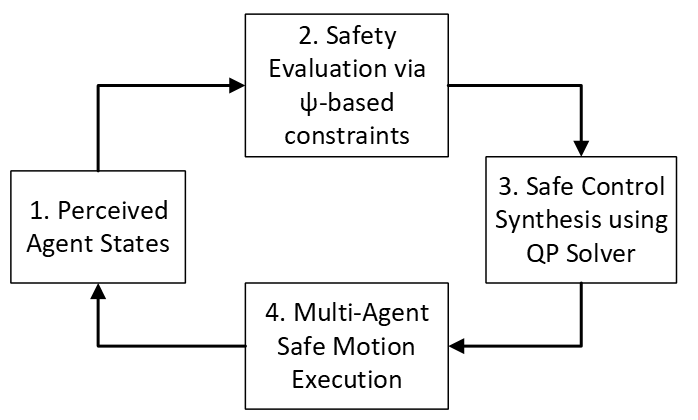}  % Adjust filename and path
  \caption{Functional states of the closed-loop multi-agent system $\mathcal{S}$}
  \label{fig:closed_loop_states}
\end{figure}

\begin{problem}[Stochastic Multi-Agent Safety Verification]
\label{prob:stochastic_safety}
For the closed-loop multi-agent system $\mathcal{S}$ where each agent pair $(i, j)$ is associated with a continuously differentiable $\psi$-weighted pairwise safety function expressed in Definition \eqref{pairwise_safety_function}, design a feedback policy $u_i = \pi_i(x)$ for each agent such that the following holds with a high probability: 
\[
\mathbb{P} \left[ \tilde{h}_{ij}(x(t), u(t)) \geq 0, \ \forall t \in [0, T], \ \forall i \neq j \right] \geq 1 - \delta,
\]
where $\delta \in (0,1)$ is a user-defined risk tolerance parameter.

\end{problem}

\section{Solution Approach}

The proposed hybrid verification approach consists of the following steps:
\begin{enumerate}
    \item Derive the formal deterministic safety guarantees using CBFs and the admissible control set, ensuring forward invariance of the safe set in the absence of noise.
    \item Extend step 1 to bounded stochastic settings by performing Monte Carlo rollouts with sampled disturbances, and estimate margin-aware safety violations based on observed behavior.
    \item Finally, derive a probabilistic upper bound on the true violation probability using PAC-style generalization, wherein the bound holds with a high probability over unseen realizations.
\end{enumerate}

\subsection{Safety via $\psi$-Weighted CBFs}

The safety of closed-loop multi-agent systems using CBFs is closely related to the notion of forward invariance, which is formally analyzed using CBFs extended with the $\psi$-weighted alignment term. 

A safe set characterizes the subset of the state space in which an agent can operate without violating safety constraints.
\begin{definition}[Safe Set]
For each agent $i$, the safe set is defined as:
\begin{equation}
\mathcal{C}_i \triangleq \{x_i \in \mathbb{R}^n \mid h_i(x_i) \geq 0\},
\end{equation}
where $h_i: \mathbb{R}^n \to \mathbb{R}$ is a continuously differentiable function encoding the safety condition.
\end{definition}

\begin{definition}[Forward Invariance of a Safe Set]
A safe set $\mathcal{C}_i$ is forward invariant with respect to the closed-loop dynamics of agent $i$: \(
\dot{x}_i = f(x_i) + g(x_i) u_i + w_i(t),
\)
if for every initial condition $x_i(0) \in \mathcal{C}_i$, the system state satisfies $x_i(t) \in \mathcal{C}_i, \text{ for all time }t \geq 0$.
\end{definition}

CBFs are employed to enforce forward invariance, and a valid CBF for the safe set $\mathcal{C}_i$ is defined as:

\begin{definition}[Valid CBF]
A continuously differentiable function $h_i : \mathbb{R}^n \to \mathbb{R}$ is a valid CBF for the safe set $\mathcal{C}_i$ if there exists an extended class-$\mathcal{K}$ function $\alpha : \mathbb{R} \to \mathbb{R}$ such that:
\begin{equation}
\label{eq:CBF}
\dot{h}_i(x_i,u_i) + \alpha(h_i(x_i)) \geq 0,
\end{equation}
for all admissible state-control pairs $(x_i,u_i)$.
\end{definition}

\begin{lemma}[Single-Agent Safety via CBF]
\label{lem:single_agent_safety}
Given a valid CBF $h_i$ for agent $i$ and corresponding safe set $\mathcal{C}_i$, any control input $u_i$ satisfying: \(
\dot{h}_i(x_i,u_i) + \alpha(h_i(x_i)) \geq 0,
\)
for all $x_i \in \mathcal{C}_i$, ensures the forward invariance of the set $\mathcal{C}_i$.
\end{lemma}

\textit{Proof:} This is well-established and we refer to Comparison Principle (Lemma 3.4 from \cite{khalil_nonlinear_2002}) for the proof. 

For multi-agent scenarios, the single-agent CBF (Equation \eqref{eq:CBF}) is extended to the pairwise setting using the $\psi$-weighted safety function defined in Equation \eqref{eq:safety_function}, and a QP is solved to satisfy the condition in Lemma \ref{lem:single_agent_safety}. Next, the admissible control set to guarantee safety across multi-agent interactions is defined.

\begin{definition}[$\psi$-Weighted Admissible Control Set]
\label{def:psi_weighted_admissible_control_set}
The admissible control set $\bar{\mathcal{U}}(x)$ is defined as:
\begin{equation}
\bar{\mathcal{U}}(x) = \left\{u \in \mathbb{R}^{Nm} \,\middle|\, \dot{\tilde{h}}_{ij}(x,u) + \alpha(\tilde{h}_{ij}(x,u)) \geq 0, \forall i \neq j\right\},
\end{equation}
where $\dot{\tilde{h}}_{ij}(x, u)$ denotes the time derivative of the $\psi$-weighted safety function, computed along the system dynamics, and $\alpha(\cdot)$ is the extended class-$\mathcal{K}$ function. 
\end{definition}

This admissible set comprises the collection of all joint control inputs that ensure the forward invariance of every pairwise safe set, thus preserving safety of the full system under both spatial and directional interactions.

\begin{proposition}[Admissible Control Set Non-Empty Condition]
\label{prop:admissibile_control_set}
If at state $x \in \mathbb{R}^{Nn}$, there exists a control input $u$ satisfying: \(
\dot{\tilde{h}}_{ij}(x,u) + \alpha(\tilde{h}_{ij}(x,u)) \geq 0,
\)
for all pairs $(i,j)$, then the admissible control set $\bar{\mathcal{U}}(x)$ is non-empty.
\end{proposition}

Based on Proposition \ref{prop:admissibile_control_set}, we can now assert the forward invariance of pairwise safety constraints.

\begin{lemma}[Deterministic Forward Invariance of Pairwise $\psi$-Weighted Safe Sets]
\label{lemma:forward_invariance}
Given a pair of distinct agents $(i, j)$ in the closed-loop multi-agent system $\mathcal{S}$ and $\psi$-weighted pairwise safety function $\tilde{h}_{ij}(x, u)$ defined in Equation~\eqref{eq:safety_function}, if the initial condition satisfies $\tilde{h}_{ij}(x(0), u(0)) \geq 0$, and the control input satisfies the admissibility condition $u(t) \in \bar{\mathcal{U}}(x(t))$ for all $t \geq 0$, then the set:
\[
\tilde{\mathcal{C}}_{ij} = \left\{ x \in \mathbb{R}^{Nn} \mid \tilde{h}_{ij}(x, u) \geq 0 \right\}
\]
is forward invariant under the nominal (noise-free) dynamics.
\end{lemma}

\begin{proof}
Let $x_i(t), x_j(t) \in \mathbb{R}^n$ denote the states of agents $i$ and $j$ respectively, evolving according to the control-affine dynamics in Equation \eqref{eq:dynamics}, but under noise-free conditions. Therefore, We get:
\[
\dot{x}_i = f(x_i) + g(x_i) \cdot u_i, \quad \dot{x}_j = f(x_j) + g(x_j) \cdot u_j.
\]

Differentiating $\tilde{h}_{ij}$ along the noise-free system trajectories yields:
\[
\dot{\tilde{h}}_{ij}(x,u) = \dot{h}_{ij}(x) + \psi \cdot \frac{d}{dt} \left[ \mathcal{A}_{ij}^\top (u_i - u_j) \right].
\]
The term $\dot{h}_{ij}(x)$ follows from the chain rule as:
\(
\dot{h}_{ij}(x) = \nabla h_{ij}(x)^\top \left( \dot{x}_i - \dot{x}_j \right),
\)
and substituting the dynamics gives:
\begin{align}
\dot{h}_{ij}(x) = \nabla h_{ij}(x)^\top \Big[\, &(f(x_i) + g(x_i) \cdot u_i) - (f(x_j)  + g(x_j) \cdot u_j) \nonumber \,\Big] \nonumber.
\end{align}

Now, by the admissibility condition, the control input satisfies:
\(
\dot{\tilde{h}}_{ij}(x(t), u(t)) + \alpha(\tilde{h}_{ij}(x(t), u(t))) \geq 0, \forall t \geq 0,
\)
for some extended class-$\mathcal{K}$ function \( \alpha(\cdot) \).

Let \( z(t) \) be a scalar function that evolves according to the differential equation:
\[
\dot{z}(t) = -\alpha(z(t)), \quad z(0) = \tilde{h}_{ij}(x(0), u(0)) \geq 0.
\]

Then, by the Comparison Principle (see Lemma 3.4 in \cite{khalil_nonlinear_2002}), the solution \( \tilde{h}_{ij}(x(t), u(t)) \) satisfies:
\[
\tilde{h}_{ij}(x(t), u(t)) \geq z(t), \forall t \geq 0.
\]

Since \( z(t) \geq 0 \) for all \( t \geq 0 \), it follows that:
\[
\tilde{h}_{ij}(x(t), u(t)) \geq 0, \forall t \geq 0.
\]
Therefore, the system trajectory remains within the safe set \( \tilde{\mathcal{C}}_{ij} \), establishing its forward invariance.
\end{proof}

Lemma \ref{lemma:forward_invariance} establishes the forward invariance of the safe set associated with a single agent pair under admissible control. To extend this result to the entire multi-agent system, we now consider the joint admissibility condition across all interacting pairs.

\begin{theorem}[Forward Invariance and Safety Guarantee]
Consider the multi-agent system $\mathcal{S}$, pairwise $\psi$-weighted safety function (Definition \eqref{pairwise_safety_function}), and initial conditions $x(0)$ satisfying $\tilde{h}_{ij}(x(0),u(0)) \geq 0$ for all pairs $(i,j)$. If admissible control inputs $u(t) \in \bar{\mathcal{U}}(x(t))$ are maintained, the system guarantees:
\[
\tilde{h}_{ij}(x(t),u(t)) \geq 0, \quad \forall i \neq j, \forall t \geq 0.
\]
\end{theorem}

\textit{Proof:} This result follows directly from Lemma 2, applied simultaneously to all pairs.

While deterministic control barrier conditions ensure forward invariance in the absence of disturbances for all time, real-world systems are rarely noise-free. To certify safety under stochastic uncertainty, this study extends the deterministic admissibility to bounded stochasticity by quantifying the probability of safety violations using PAC-style guarantees.

\section{Stochastic Safety Guarantees with Pairwise $\psi$-Weighted Control}
This section introduces the formulation of probabilistic safety for the closed-loop multi-agent system $\mathcal{S}$, providing finite-time safety guarantees under stochastic disturbances.

We define the stochastic extension of the pairwise safe set:
\begin{definition}[Finite-Time Stochastic Safety Requirement]
\label{def:finite_time_stochasticity_requirement}
    Let $\tilde{h}_{ij}(x,u)$ be the $\psi$-weighted pairwise safety function. For a given risk threshold $\delta \in (0,1)$ and finite time horizon $T > 0$, the pairwise safe set is preserved stochastically if:
\[
\mathbb{P} \left[ \tilde{h}_{ij}(x(t), u(t)) \geq 0, \ \forall t \in [0,T] \right] \geq 1 - \delta.
\]
\end{definition}

The goal is to ensure that the multi-agent system $\mathcal{S}$ remains in the intersection of all pairwise sets over $[0, T]$ with a high-probability, as expressed by the following theorem.

\begin{theorem}[Stochastic Multi-Agent Safety Guarantee]
\label{thm:stochastic_safety}
Within the closed-loop multi-agent system $\mathcal{S}$ with $N$ agents, suppose for each agent pair $i, j$ $(i \neq j)$, the $\psi$-weighted pairwise safety function $~\tilde{h}_{ij}(x, u)$ satisfies the forward invariance condition: 
\begin{equation}
\label{eq:marginal_forward_invariance}
\dot{\tilde{h}}_{ij}(x(t), u(t)) + \alpha(\tilde{h}_{ij}(x(t), u(t))) \geq \gamma_{ij}(t), \quad \forall t \in [0, T],    
\end{equation}
where $\alpha$ is a class-$\mathcal{K}$ function, and $\gamma_{ij}(t)$ is a disturbance margin satisfying:
\begin{equation}
\label{eq:disturbance_margin}
\gamma_{ij}(t) \geq \sup_{\|w_i(t)\|,\|w_j(t)\| \leq \bar{w}} \left| \nabla h_{ij}(x(t))^{\top} (w_i(t) - w_j(t)) \right|.    
\end{equation}
Then the probability that any safety constraint is violated over $[0, T]$ is upper bounded by $\delta$, i.e.,
\[
\mathbb{P} \left[ \exists \, t \in [0,T],\ \exists \, i \neq j : \tilde{h}_{ij}(x(t), u(t)) < 0 \right] \leq \delta.
\]
\end{theorem}

\begin{proof}
    Let $x_i(t) \in \mathbb{R}^n$ be the state of an agent $i$ from $\mathcal{S}$ that evolves according to Equation \eqref{eq:dynamics} with bounded stochastic disturbance $\|w_i(t)\| \leq \bar{w}$. Define $Z_{ij}(t) := \tilde{h}_{ij}(x(t), u(t))$. From the assumptions in Equations \eqref{eq:marginal_forward_invariance} and \eqref{eq:disturbance_margin}, 
    \begin{align}
    \dot{Z}_{ij}(t) + \alpha(Z_{ij}(t)) 
    &\geq \gamma_{ij}(t) \nonumber \\
    \geq\ &\sup_{\|w_i(t)\|,\, \|w_j(t)\| \leq \bar{w}} 
    \left| \nabla h_{ij}(x(t))^\top \big( w_i(t) - w_j(t) \big) \right|.
\end{align}

which implies that the margin of increase in $Z_{ij}(t)$ caused by the deterministic part (drift term) is at least as large as the negative contribution induced by the stochastic noise, offsetting the push toward violation. 

Next, let the time interval $[0, T]$ be discretized into $K$ uniform steps of size $\Delta t := T / K$. Define the discrete time points as $t_k := k \cdot \Delta t$ for $0 \leq k \leq K$. At each step, define $Z_{ij}(t_k) := \tilde{h}_{ij}(x(t_k), u(t_k))$. From the Lipschitz continuity of $f$, $g$, and $h_{ij}$, together with the bound $\|w_i(t)\| \leq \bar{w}$, there exists a constant $c > 0$ such that:
\begin{equation}
\label{eq:increment_bound}
|\Delta Z_{ij}^{(k)}| \leq c, \quad \forall k.
\end{equation}

Furthermore, the disturbance term in \eqref{eq:disturbance_margin} ensures that the conditional variance of $\Delta Z_{ij}^{(k)}$ given the past, is bounded as:
\begin{equation}
\label{eq:variance_bound}
\mathrm{Var}\!\left[ \Delta Z_{ij}^{(k)} \mid \mathcal{F}_k \right] \leq \sigma^2, \quad \forall k,
\end{equation}
where $\mathcal{F}_k$ is the natural filtration of the process encompassing all control inputs and noise realizations up to time $t_k$, and
\begin{equation}
\label{eq:variance_sigma}
\sigma^2 := 4 \bar{w}^2 \, \sup_{x \in \mathbb{R}^{Nn}} \| \nabla h_{ij}(x) \|^2 \, (\Delta t)^2.
\end{equation}

We now apply the Bernstein inequality for sums of bounded, zero-mean random variables with bounded variance.  
Define the centered increments:
\[
Y_k := \Delta Z_{ij}^{(k)} - \mathbb{E}[\Delta Z_{ij}^{(k)} \mid \mathcal{F}_k].
\]
By construction, $(Y_k)_{k=0}^{K-1}$ is a martingale difference sequence with $|Y_k| \leq c$ and $\mathrm{Var}[Y_k \mid \mathcal{F}_k] \leq \sigma^2$.

Bernstein's (one-sided) inequality then states that for any $\epsilon > 0$:
\begin{equation}
\label{eq:bernstein_bound}
\mathbb{P} \left[ \sum_{k=0}^{K-1} Y_k \leq -\epsilon \right] 
\leq \exp \left( - \frac{\epsilon^2}{2 K \sigma^2 + \frac{2}{3} c \epsilon} \right).
\end{equation}

Let $h_{\min} := \min_{i \neq j} Z_{ij}(0)$ be the minimum initial safety margin.  
A violation occurs if $\sum_{k=0}^{K-1} Y_k \leq -h_{\min}$ at some step, which would drive $Z_{ij}(t_k)$ below zero.  
Setting $\epsilon = h_{\min}$ in \eqref{eq:bernstein_bound} gives:
\[
\mathbb{P} \left[ \min_{0 \leq k \leq K} Z_{ij}(t_k) < 0 \right] 
\leq \exp \left( - \frac{h_{\min}^2}{2 K \sigma^2 + \frac{2}{3} c h_{\min}} \right) 
\triangleq \delta_{ij}.
\]

Considering the presence of $N$ agents in $\mathcal{S}$, there are a total of $\binom{N}{2}$ unique agent pairs. Applying the union bound over all such pairs yields:
\[
\mathbb{P} \left[ \exists\, t \in [0,T],\ \exists\, i \neq j : \tilde{h}_{ij}(x(t), u(t)) < 0 \right] \leq \sum_{i < j} \delta_{ij} \leq \delta,
\]
provided the initial margin $h_{\min} > 0$ and the robustness margin $\gamma_{ij}(t)$ exceeds the worst-case noise effect. Thus, with probability at least $1 - \delta$, the system satisfies all pairwise safety constraints throughout the interval $[0, T]$, satisfying the requirement in Definition \ref{def:finite_time_stochasticity_requirement}.
\end{proof}

\section{Empirical Validation via Margin-Aware Monte Carlo Safety Bounds}

This section presents an empirical framework to validate the formal stochastic safety guarantees established in Theorem~\ref{thm:stochastic_safety}. This study moves beyond traditional binary violation analysis and introduces a margin-aware notion of safety that captures not only whether constraints are violated, but by how closely they are approached. The system's robustness is empirically evaluated under bounded stochastic disturbances by simulating $P$ Monte Carlo rollouts, each with randomized noise realizations. For each rollout, a normalized safety margin is computed that reflects the worst-case proximity to constraint violation over time and across agent pairs. This fine-grained, real-valued metric enables the derivation of probabilistic generalization guarantees via classical concentration inequalities such as the Bernstein bound, yielding PAC-style bounds on the true safety performance.

Let $P \in \mathbb{N}$ be the number of independent closed-loop simulations (scenarios) of $\mathcal{S}$ with $N$ agents, each subject to a distinct realization of bounded noise $\{w_i^{(p)}(t)\}_{i=1}^N\}$ with $\|w_i^{(p)}(t)\| \leq \bar{w}$ where $p \in \{1, \dots, P\}$. Each trajectory begins from an initial condition where all agent pairs $(i, j)$ satisfy the $\psi$-weighted CBF constraint $\tilde{h}_{ij}(x_i(0), x_j(0)) \geq 0$. During each rollout $p$, all pairwise safety constraints $\tilde{h}_{ij}(x^{(p)}(t), u^{(p)}(t))$ over the finite time horizon $t \in [0, T]$, where $i \neq j \in \{1, \dots, N\}$ are monitored. Instead of recording binary violation outcomes, the minimum safety margin encountered over the entire rollout is recorded. It is defined as:
\begin{equation}
\label{eq:margin_aware_safety}
Z_p := 
\begin{cases}
0, & \text{if } \mathcal{V}^{(p)} \text{ occurs} \\
\displaystyle \min_{t \in [0, T],\, i \neq j} 
\left[ \frac{\tilde{h}_{ij}(x^{(p)}(t), u^{(p)}(t))}{\max(\tilde{h}_{ij}^{\max}, \epsilon)} \right], 
& \text{otherwise}
\end{cases}
\end{equation}
where \(
\mathcal{V}^{(p)} := \left\{ \exists\, t \in [0, T],\, \exists\, i \neq j:\ \tilde{h}_{ij}(x^{(p)}(t), u^{(p)}(t)) < 0 \right\}
\) is the violation event for rollout $p$,
$\tilde{h}_{ij}^{\max}$ is the maximum margin observed across all $P$ rollouts, and $\epsilon > 0$ is a small constant to avoid division by zero.

Based on a variable safety violation threshold $\theta \in (0, 1)$, a rollout is classified as violating safety if its margin-aware score $Z_p$ falls below the threshold i.e., $Z_p < \theta$. Formally, define the binary indicator for each rollout:
\[
X_p := \mathbb{1}[Z_p < \theta], \quad \forall\, p \in \{1, \dots, P\}
,
\]
where $X_p = 1$ indicates a safety violation, and $X_p = 0$ indicates no violation. 

The empirical violation rate across $P$ rollouts is given by: 
\begin{equation}
\label{eq:empirical_mean}
\hat{p}_v := \frac{1}{P} \sum_{p=1}^P X_p,     
\end{equation}
and the empirical pairwise variance is given by:
\begin{equation}
\label{eq:empirical_variance}
\hat{\sigma}^2 := \frac{1}{P(P - 1)} \sum_{1 \le i < j \le P} (X_i - X_j)^2.    
\end{equation}

The collection $\{X_p\}_{p=1}^P$ is modeled as a sequence of independently and identically distributed (i.i.d.) samples from a Bernoulli distribution with an unknown probability $p^* := \mathbb{P}[Z < \theta]$. The following PAC-style bound provides a high-confidence estimate of this true violation probability.

\begin{theorem}[Empirical Bernstein Violation Bound for Safety Margins]
\label{thm:bernstein_violation_bound}
Consider a closed-loop multi-agent system $\mathcal{S}$ with $N$ agents operating under bounded stochastic disturbances. Let $(Z_1, \dots, Z_P) \in [0, 1]$ denote the margin-aware safety scores computed over $P$ independent Monte Carlo rollouts, where $\theta \in (0, 1)$ is the violation threshold. Then, with probability at least $1 - \delta$, the true violation probability satisfies:
\[
\mathbb{P}\left[ Z < \theta \right] \leq \hat{p}_v + \sqrt{ \frac{2 \hat{\sigma}^2 \log(2/\delta)}{P} } + \frac{7 \log(2/\delta)}{3(P - 1)}.
\]
\end{theorem}

\begin{proof}
    Let $(X_1, \dots, X_p)$ be i.i.d. real-valued random variables bounded in $[0, 1]$, with true mean $\mu := \mathbb{E}[X]$ and empirical mean $\hat{\mu} := \frac{1}{P} \sum_{i=1}^P X_i$. The classical Bernstein inequality for zero-mean, independent random variables $X_i$ and bounded deviations $|X_i| \le D$ almost surely, is shown in Equation \eqref{eq:bernstein_classical}
    \begin{equation}
    \label{eq:bernstein_classical}
    \mathbb{P} \left( \left| \sum_{i=1}^{P} X_i \right| \ge \epsilon \right) \le 2 \exp \left( -\frac{\epsilon^2}{2 P \sigma^2 + \frac{2}{3} D \epsilon} \right).
    \end{equation}
    It requires knowledge of the true variance $\sigma^2$. However, this is unknown in practice. Therefore, we invoke the empirical Bernstein inequality that has been proposed in Theorem 11 \cite{maurer2009empirical}, which relies on empirical variance. Assuming the empirical pairwise variance $\hat{\sigma}^2$ (Equation \eqref{eq:empirical_variance}) satisfies $\sigma^2 \in [0, 1]$ when $X_p \in [0, 1]$, then for any $\delta \in (0, 1)$, with probability at least $1 - \delta$, the following holds:
    \[
\mathbb{E}[X] \leq \hat{\mu} + \sqrt{ \frac{2 \hat{\sigma}^2 \log(2/\delta)}{P} } + \frac{7 \log(2/\delta)}{3(P - 1)}.
\]

In our setting, $(X_p)_{p=1}^P$ is a sequence of i.i.d. random variables where $X_p \in [0, 1]$ for each rollout $p$. Therefore, the true violation probability can be expressed as the expected value of $X_p$, since $\mathbb{E}[X_p] = \mathbb{E}[\mathbb{1}[Z_p < \theta]] = \mathbb{P}[Z_p < \theta]$ 
i.e., 
\[
\mu := \mathbb{E}[X_p] = \mathbb{P}[Z < \theta].
\]
The empirical violation rate $p_v$ and empirical variance $\hat{\sigma}^2$ are as defined in Equations \eqref{eq:empirical_mean} and \eqref{eq:empirical_variance} respectively. Therefore, with probability at least $1 - \delta$, the true violation probability is bounded as:
\[
\mathbb{P}\left[ Z < \theta \right] \leq \hat{p}_v + \sqrt{ \frac{2 \hat{\sigma}^2 \log(2/\delta)}{P} } + \frac{7 \log(2/\delta)}{3(P - 1)}.
\]

\end{proof}

\section{Hybrid Deterministic–Empirical Safety Guarantee}

The deterministic admissibility is combined with empirical validation under stochastic disturbances to provide a unified PAC-style safety guarantee, as follows.

\begin{theorem}
\label{thm:hybrid_safety}
For a closed-loop multi-agent system $\mathcal{S}$ controlled by a $\psi$-weighted CBF where:
\begin{enumerate}
    \item The initial state $x(0)$ satisfies all pairwise CBF conditions: $\tilde{h}_{ij}(x(0), u(0)) \geq h_{\min} > 0$ for all $i \neq j$.
    \item The deterministic admissibility condition holds: for each pair $(i, j)$, the control input $u(t)$ ensures
    \[
    \dot{\tilde{h}}_{ij}(x(t), u(t)) \geq -\alpha(\tilde{h}_{ij}(x(t), u(t))) + \gamma_{ij}(t),
    \]
    where \[\gamma_{ij}(t) \geq \sup_{\|w_i(t)\|,\|w_j(t)\| \leq \bar{w}} \left| \nabla h_{ij}(x(t))^{\top} (w_i(t) - w_j(t)) \right|\] exceeds the worst-case disturbance effect.
    \item $P$ independent Monte Carlo rollouts are performed under bounded stochastic disturbances. Let $(Z_p)_{p=1}^P \in [0,1]$ denote the normalized safety margins for each rollout, and let $\theta \in (0, 1)$ be a user-defined violation threshold. 
\end{enumerate}
Then, with probability at least $1 - \delta$ over future disturbance realizations, the system satisfies:
\[
\mathbb{P}\left[ Z < \theta \right] \leq \hat{p}_v + \sqrt{ \frac{2 \hat{\sigma}^2 \log(2/\delta)}{P} } + \frac{7 \log(2/\delta)}{3(P - 1)}.
\]
\end{theorem}

\begin{proof}
    The proof for the theorem follows directly from the established behavior of the system under nominal (noise-free) dynamics and stochasticity in prior theorems. The required steps for the safety certificate are provided in Algorithm \ref{alg:margin_aware_safety_verification}. 
\end{proof}

\begin{algorithm}[H]
\caption{Margin-Aware Hybrid Safety Verification}
\label{alg:margin_aware_safety_verification}
\begin{algorithmic}[1]
\Require Multi-agent system $\mathcal{S}$ with dynamics \eqref{eq:dynamics}, $\psi$-weighted CBF control, rollouts $P$, violation threshold $\theta \in (0, 1)$, confidence level $\delta \in (0,1)$
\Ensure Violation probability bound $\epsilon_S$ such that with probability $\geq 1 - \delta$, the system is safe on $[0, T]$
\State Initialize binary indicators $X_p \gets 0$ for $p = 1$ to $P$
\For{$p = 1$ to $P$}
  \State Sample $x^{(p)}(0)$ such that $\tilde{h}_{ij}(x^{(p)}(0), u^{(p)}(0)) \geq h_{\min}$ for all $i \neq j$
  \State Sample bounded noise $w^{(p)}(t)$ for $t \in [0, T]$
  \State Simulate $\mathcal{S}$ under $w^{(p)}(t)$ using $u^{(p)}(t)$
  \State Compute normalized minimum safety margin: $Z_p := \min\limits_{t \in [0, T],\, i \neq j} \frac{\tilde{h}_{ij}(x^{(p)}(t), u^{(p)}(t))}{\max(\tilde{h}_{ij}^{\max}, \epsilon)}$
  \If{$Z_p < \theta$}
    \State $X_p \gets 1$ \Comment{Mark rollout as violating}
  \EndIf
\EndFor
\State $\hat{p}_v := \frac{1}{P} \sum_{p=1}^{P} X_p$ \Comment{Empirical violation rate}
\State $\hat{\sigma}^2 := \frac{1}{P(P - 1)} \sum_{1 \le i < j \le P} (X_i - X_j)^2$ \Comment{Empirical pairwise variance}
\State $\epsilon_S := \hat{p}_v + \sqrt{ \frac{2 \hat{\sigma}^2 \log(2/\delta)}{P} } + \frac{7 \log(2/\delta)}{3(P - 1)}$
\State \Return $\epsilon_S$ as the upper bound on true violation probability
\end{algorithmic}
\end{algorithm}

\section{Experiments and Results}

\subsection{Verification of the Proposed Hybrid Safety Guarantee}
\label{sec:verification_of_safety_guarantee}

The proposed hybrid safety verification framework is evaluated on a closed-loop multi-agent system $\mathcal{S}$ governed by control-affine dynamics (Equation~\eqref{eq:dynamics}) and subject to bounded stochastic disturbances. Three thresholds of bounded noise magnitude are considered: $\bar{w} \in \{0.01, 0.03, 0.05\}$. For each noise setting, $100$ independent groups, each comprising $50$ stochastic rollouts are executed. Within each rollout, the system evolves for a finite horizon of $T = 50$ time steps with a fixed integration step size $\Delta t = 0.1$. At each time step, the control input for all agents is computed by solving a constrained QP that minimizes the instantaneous control effort while satisfying all pairwise safety constraints based on the proposed $\psi$-weighted CBF formulation. Initial agent configurations are sampled uniformly within a bounded spatial domain, with a hard constraint ensuring a minimum pairwise separation of one unit at $t = 0$, so that all agents begin in a safe state. A violation threshold $\theta = 0.1$, is set to detect safety violations at any step during the rollout. For each group of rollouts, the empirical violation probability $\hat{p}_v$ is recorded, and the corresponding empirical Bernstein bound is computed using the proposed approach. Additionally, we compute the classical safety bounds for comparison, including the Hoeffding bound~\cite{maurer2009empirical} and the scenario-based bound~\cite{doi:10.1137/07069821X} derived from convex optimization under i.i.d. uncertainty. The results of the experiment conducted independently for $N=2$, and $N=3$ agents with $\delta=0.1$ are presented in Table \ref{tab:verification_results}. The $N$ values are chosen to be computationally tractable for the experiment implemented using CVXPY, a Python-based library for convex optimization problems, on a Windows 11 machine with an AMD Ryzen 9950X processor. 
 
\begin{table}[ht]
\centering
\caption{Verification Bound Satisfaction Results}
\scriptsize
\begin{tabular}{c}
\textbf{Two Agents} \\
\end{tabular}

\begin{tabular}{|c|c|c|c|c|c|c|c|c|}
\hline
$\bar{w}$ & $\hat{p}$ & $\epsilon_B$ & $\epsilon_H$ & $\epsilon_S$ & $B_{sat}$ & $H_{sat}$ & $S_{sat}$ \\
\hline
0.01 & 0.028 & 0.150 & 0.173 & 0.046 & 1.00 & 1.00 & 0.84 \\
0.03 & 0.055 & 0.154 & 0.173 & 0.046 & 0.99 & 1.00 & 0.45 \\
0.05 & 0.091 & 0.157 & 0.173 & 0.047 & 0.93 & 0.95 & 0.13 \\
\hline
\end{tabular}

\begin{tabular}{c}
\textbf{Three Agents} \\
\end{tabular}
\begin{tabular}{|c|c|c|c|c|c|c|c|c|}
\hline
$\bar{w}$ & $\hat{p}$ & $\epsilon_B$ & $\epsilon_H$ & $\epsilon_S$ & $B_{sat}$ & $H_{sat}$ & $S_{sat}$ \\
\hline
0.01 & 0.065 & 0.155 & 0.173 & 0.046 & 0.96 & 1.00 & 0.16 \\
0.03 & 0.083 & 0.162 & 0.173 & 0.046 & 0.94 & 0.96 & 0.01 \\
0.05 & 0.102 & 0.162 & 0.173 & 0.046 & 0.94 & 0.95 & 0.00 \\
\hline
\end{tabular}

\label{tab:verification_results}
\end{table}

{\scriptsize
\noindent $\hat{p}$: Average empirical violation rate across 100 groups of $P = 50$ rollouts each, where each group's violation rate is computed using Equation \ref{eq:empirical_mean}. $\epsilon_H = \sqrt{ \frac{\log(2/\delta)}{2P} }$: Hoeffding bound per group (identical across groups). $\epsilon_B = \sqrt{ \frac{2 \hat{\sigma}^2 \log(2/\delta)}{P} } + \frac{7 \log(2/\delta)}{3(P - 1)}$: Empirical Bernstein bound per group (averaged across groups). $\epsilon_S = \frac{d_k + \log(1/\delta)}{P}$: Scenario-based bound per group (averaged across groups), where $d_k$ is the number of support constraints in group $k$. $B_{\text{sat}}, H_{\text{sat}}, S_{\text{sat}}$: Fraction of groups where the corresponding bound was satisfied.
}

On the empirical Bernstein-based bound $\epsilon_B$ derived from the proposed hybrid formulation, $B_{\text{sat}} \geq 0.9$ in all cases across both agent configurations. This is in alignment with the expected probabilistic safety guarantee of at least $1-\delta \geq 0.9$. The Hoeffding bound $H_{sat}$ is satisfied more than $B_{\text{sat}}$ showing that it is looser than the empirical Bernstein bound. This is expected as it doesn't account for pairwise variance between agents during rollouts. As the number of support constraints is low, scenario-based bounds become excessively conservative resulting in the satisfaction rate of $S_{sat}$ being low. Overall, the results validate the soundness of the proposed approach and also demonstrates that it is a sharp, data-driven technique to certify multi-agent safety under bounded noise.

\subsection{$\psi$-Sensitivity Analysis}
To demonstrate the impact of $\psi$ which is not a part of CBFs typically, a sensitivity analysis is conducted varying $\psi \in \{0, 2, 4, 6, 8, 10 \}$ across 100 randomized rollouts per $\psi$ value on $\mathcal{S}$ with noise bounded at $\bar{w} = 0.03$. The setting of the experiment is similar to \ref{sec:verification_of_safety_guarantee}. The empirical violation rate ($\hat{p}_{v}$), and minimum distance (Min. Dist.) between the agents across the 100 trials are presented in Table \ref{tab:psi_ablation}.

\begin{table}[ht]
\centering
\caption{$\psi$-Sensitivity Study: Safety Metrics}
\scriptsize
\begin{tabular}{cc}
\begin{tabular}{|c|c|c|}
\multicolumn{3}{c}{\textbf{Two Agents}} \\
\hline
$\psi$ & $\hat{p}_{v}$ & Min. Dist. \\
\hline
0  & 0.09 & 7.03 \\
2  & 0.07 & 7.34 \\
4  & 0.07 & 7.60 \\
6  & 0.04 & 7.79 \\
8  & 0.03 & 7.43 \\
10 & 0.03 & 7.31 \\
\hline
\end{tabular}
&
\begin{tabular}{|c|c|c|}
\multicolumn{3}{c}{\textbf{Three Agents}} \\
\hline
$\psi$ & $\hat{p}_{v}$ & Min. Dist. \\
\hline
0  & 0.20 & 2.12 \\
2  & 0.18 & 2.41 \\
4  & 0.17 & 2.52 \\
6  & 0.17 & 2.39 \\
8  & 0.15 & 2.38 \\
10 & 0.14 & 2.41 \\
\hline
\end{tabular}
\end{tabular}
\label{tab:psi_ablation}
\end{table}

\vspace{-0.5em}
{\scriptsize
\noindent  $\hat{p}_{v}$: from Equation \eqref{eq:empirical_mean}. \quad
Min. Dist.: $= \min_{t \in [0, T]} \min_{i \neq j} \|x_i(t) - x_j(t)\|_2$ where $x_i$ and $x_j$ are distinct agents.
}

The results demonstrate that increasing the $\psi$ coefficient yields a consistent reduction in $\hat{p}_v$, indicating improved safety performance. Min. Dist., which indicates the average worst-case spatial proximity across 100 rollouts is comfortably above the minimum pairwise separation of one unit indicating that the agents safely navigate through the environment, barring the violations. These results validate that incorporating $\psi$ enhances safety in closed-loop multi-agent systems without compromising minimum safe distances. 

\section{Conclusion and Future Work}
This work presents a hybrid safety framework that combines deterministic admissibility with PAC-style safety bounds to certify closed-loop multi-agent systems under bounded stochastic disturbances. Unlike classical concentration bounds that are based solely on violation frequency, the proposed approach introduces margin-aware safety across sampled rollouts to derive a data-driven upper bound on violation probability, bridging deterministic control theory with empirical robustness. Future work will focus on extending the approach to high-dimensional systems with non-convex safe sets and implementing the proposed approach in a real-world case study to determine scalability.

\bibliography{conference_101719}

\begin{thebibliography}{10}
\providecommand{\url}[1]{#1}
\csname url@samestyle\endcsname
\providecommand{\newblock}{\relax}
\providecommand{\bibinfo}[2]{#2}
\providecommand{\BIBentrySTDinterwordspacing}{\spaceskip=0pt\relax}
\providecommand{\BIBentryALTinterwordstretchfactor}{4}
\providecommand{\BIBentryALTinterwordspacing}{\spaceskip=\fontdimen2\font plus
\BIBentryALTinterwordstretchfactor\fontdimen3\font minus \fontdimen4\font\relax}
\providecommand{\BIBforeignlanguage}[2]{{%
\expandafter\ifx\csname l@#1\endcsname\relax
\typeout{** WARNING: IEEEtran.bst: No hyphenation pattern has been}%
\typeout{** loaded for the language `#1'. Using the pattern for}%
\typeout{** the default language instead.}%
\else
\language=\csname l@#1\endcsname
\fi
#2}}
\providecommand{\BIBdecl}{\relax}
\BIBdecl

\bibitem{fisher2021towards}
M.~Fisher, V.~Mascardi, K.~Y. Rozier, B.-H. Schlingloff, M.~Winikoff, and N.~Yorke-Smith, ``Towards a framework for certification of reliable autonomous systems,'' \emph{Autonomous Agents and Multi-Agent Systems}, vol.~35, no.~1, p.~8, 2021.

\bibitem{araujo2023testing}
H.~Araujo, M.~R. Mousavi, and M.~Varshosaz, ``Testing, validation, and verification of robotic and autonomous systems: a systematic review,'' \emph{ACM Transactions on Software Engineering and Methodology}, vol.~32, no.~2, pp. 1--61, 2023.

\bibitem{abdelkader2021aerial}
M.~Abdelkader, S.~G{\"u}ler, H.~Jaleel, and J.~S. Shamma, ``Aerial swarms: Recent applications and challenges,'' \emph{Current robotics reports}, vol.~2, no.~3, pp. 309--320, 2021.

\bibitem{abhang2024swarm}
L.~Abhang, A.~Gummadi, R.~Changala, V.~A. Vuyyuru, I.~I. Raj \emph{et~al.}, ``Swarm intelligence for multi-robot coordination in agricultural automation,'' in \emph{2024 10th International Conference on Advanced Computing and Communication Systems (ICACCS)}, vol.~1.\hskip 1em plus 0.5em minus 0.4em\relax IEEE, 2024, pp. 455--460.

\bibitem{bi2021safety}
Z.~M. Bi, C.~Luo, Z.~Miao, B.~Zhang, W.-J. Zhang, and L.~Wang, ``Safety assurance mechanisms of collaborative robotic systems in manufacturing,'' \emph{Robotics and Computer-Integrated Manufacturing}, vol.~67, p. 102022, 2021.

\bibitem{yang2024plug}
Z.~Yang, S.~S. Raman, A.~Shah, and S.~Tellex, ``Plug in the safety chip: Enforcing constraints for llm-driven robot agents,'' in \emph{2024 IEEE International Conference on Robotics and Automation (ICRA)}.\hskip 1em plus 0.5em minus 0.4em\relax IEEE, 2024, pp. 14\,435--14\,442.

\bibitem{drew2021multi}
D.~S. Drew, ``Multi-agent systems for search and rescue applications,'' \emph{Current Robotics Reports}, vol.~2, no.~2, pp. 189--200, 2021.

\bibitem{borquez2024safety}
J.~Borquez, K.~Chakraborty, H.~Wang, and S.~Bansal, ``On safety and liveness filtering using hamilton-jacobi reachability analysis,'' \emph{IEEE Transactions on Robotics}, 2024.

\bibitem{bansal2017hamilton}
S.~Bansal, M.~Chen, S.~Herbert, and C.~J. Tomlin, ``Hamilton-jacobi reachability: A brief overview and recent advances,'' in \emph{2017 IEEE 56th Annual Conference on Decision and Control (CDC)}.\hskip 1em plus 0.5em minus 0.4em\relax IEEE, 2017, pp. 2242--2253.

\bibitem{10.1007/978-3-540-24743-2_32}
S.~Prajna and A.~Jadbabaie, ``Safety verification of hybrid systems using barrier certificates,'' in \emph{Hybrid Systems: Computation and Control}, R.~Alur and G.~J. Pappas, Eds.\hskip 1em plus 0.5em minus 0.4em\relax Berlin, Heidelberg: Springer Berlin Heidelberg, 2004, pp. 477--492.

\bibitem{4287147}
S.~Prajna, A.~Jadbabaie, and G.~J. Pappas, ``A framework for worst-case and stochastic safety verification using barrier certificates,'' \emph{IEEE Transactions on Automatic Control}, vol.~52, no.~8, pp. 1415--1428, 2007.

\bibitem{7782377}
A.~D. Ames, X.~Xu, J.~W. Grizzle, and P.~Tabuada, ``Control barrier function based quadratic programs for safety critical systems,'' \emph{IEEE Transactions on Automatic Control}, vol.~62, no.~8, pp. 3861--3876, 2017.

\bibitem{ames2019control}
A.~D. Ames, S.~Coogan, M.~Egerstedt, G.~Notomista, K.~Sreenath, and P.~Tabuada, ``Control barrier functions: Theory and applications,'' in \emph{2019 18th European control conference (ECC)}.\hskip 1em plus 0.5em minus 0.4em\relax Ieee, 2019, pp. 3420--3431.

\bibitem{qiang2024quantum}
X.~Qiang, S.~Ma, and H.~Song, ``Quantum walk computing: Theory, implementation, and application,'' \emph{Intelligent Computing}, vol.~3, p. 0097, 2024.

\bibitem{chen2021learning}
S.~Chen, M.~Fazlyab, M.~Morari, G.~J. Pappas, and V.~M. Preciado, ``Learning lyapunov functions for hybrid systems,'' in \emph{Proceedings of the 24th International Conference on Hybrid Systems: Computation and Control}, 2021, pp. 1--11.

\bibitem{clark2021verification}
A.~Clark, ``Verification and synthesis of control barrier functions,'' in \emph{2021 60th IEEE Conference on Decision and Control (CDC)}.\hskip 1em plus 0.5em minus 0.4em\relax Ieee, 2021, pp. 6105--6112.

\bibitem{alan2023control}
A.~Alan, A.~J. Taylor, C.~R. He, A.~D. Ames, and G.~Orosz, ``Control barrier functions and input-to-state safety with application to automated vehicles,'' \emph{IEEE Transactions on Control Systems Technology}, vol.~31, no.~6, pp. 2744--2759, 2023.

\bibitem{choi2021robust}
J.~J. Choi, D.~Lee, K.~Sreenath, C.~J. Tomlin, and S.~L. Herbert, ``Robust control barrier--value functions for safety-critical control,'' in \emph{2021 60th IEEE Conference on Decision and Control (CDC)}.\hskip 1em plus 0.5em minus 0.4em\relax IEEE, 2021, pp. 6814--6821.

\bibitem{10383473}
H.~Wang, A.~Papachristodoulou, and K.~Margellos, ``Distributed safety verification for multi-agent systems,'' in \emph{2023 62nd IEEE Conference on Decision and Control (CDC)}, 2023, pp. 5481--5486.

\bibitem{breeden2021high}
J.~Breeden and D.~Panagou, ``High relative degree control barrier functions under input constraints,'' in \emph{2021 60th IEEE Conference on Decision and Control (CDC)}.\hskip 1em plus 0.5em minus 0.4em\relax IEEE, 2021, pp. 6119--6124.

\bibitem{wang2021learning}
C.~Wang, Y.~Meng, Y.~Li, S.~L. Smith, and J.~Liu, ``Learning control barrier functions with high relative degree for safety-critical control,'' in \emph{2021 European Control Conference (ECC)}.\hskip 1em plus 0.5em minus 0.4em\relax IEEE, 2021, pp. 1459--1464.

\bibitem{xiao2021high}
W.~Xiao and C.~Belta, ``High-order control barrier functions,'' \emph{IEEE Transactions on Automatic Control}, vol.~67, no.~7, pp. 3655--3662, 2021.

\bibitem{wang2021safety}
C.~Wang, Y.~Meng, S.~L. Smith, and J.~Liu, ``Safety-critical control of stochastic systems using stochastic control barrier functions,'' in \emph{2021 60th IEEE Conference on Decision and Control (CDC)}.\hskip 1em plus 0.5em minus 0.4em\relax IEEE, 2021, pp. 5924--5931.

\bibitem{singletary2022safe}
A.~Singletary, M.~Ahmadi, and A.~D. Ames, ``Safe control for nonlinear systems with stochastic uncertainty via risk control barrier functions,'' \emph{IEEE Control Systems Letters}, vol.~7, pp. 349--354, 2022.

\bibitem{maghenem2021adaptive}
M.~Maghenem, A.~J. Taylor, A.~D. Ames, and R.~G. Sanfelice, ``Adaptive safety using control barrier functions and hybrid adaptation,'' in \emph{2021 American Control Conference (ACC)}.\hskip 1em plus 0.5em minus 0.4em\relax IEEE, 2021, pp. 2418--2423.

\bibitem{yang2024safe}
S.~Yang, M.~Black, G.~Fainekos, B.~Hoxha, H.~Okamoto, and R.~Mangharam, ``Safe control synthesis for hybrid systems through local control barrier functions,'' in \emph{2024 American Control Conference (ACC)}.\hskip 1em plus 0.5em minus 0.4em\relax IEEE, 2024, pp. 344--351.

\bibitem{song2022safety}
L.~Song, P.~Zhao, N.~Wan, and N.~Hovakimyan, ``Safety embedded stochastic optimal control of networked multi-agent systems via barrier states,'' \emph{arXiv preprint arXiv:2210.03855}, 2022.

\bibitem{gonzales2025multi}
M.~Gonzales, A.~Polevoy, M.~Kobilarov, and J.~Moore, ``Multi-agent feedback motion planning using probably approximately correct nonlinear model predictive control,'' \emph{arXiv preprint arXiv:2501.12234}, 2025.

\bibitem{alquier2021user}
P.~Alquier, ``User-friendly introduction to pac-bayes bounds,'' \emph{arXiv preprint arXiv:2110.11216}, 2021.

\bibitem{doi:10.1137/07069821X}
\BIBentryALTinterwordspacing
M.~C. Campi and S.~Garatti, ``The exact feasibility of randomized solutions of uncertain convex programs,'' \emph{SIAM Journal on Optimization}, vol.~19, no.~3, pp. 1211--1230, 2008. [Online]. Available: \url{https://doi.org/10.1137/07069821X}
\BIBentrySTDinterwordspacing

\bibitem{majumdar2021pac}
A.~Majumdar, A.~Farid, and A.~Sonar, ``Pac-bayes control: learning policies that provably generalize to novel environments,'' \emph{The International Journal of Robotics Research}, vol.~40, no. 2-3, pp. 574--593, 2021.

\bibitem{akella2022barrier}
P.~Akella and A.~D. Ames, ``A barrier-based scenario approach to verifying safety-critical systems,'' \emph{IEEE Robotics and Automation Letters}, vol.~7, no.~4, pp. 11\,062--11\,069, 2022.

\bibitem{khalil_nonlinear_2002}
H.~K. Khalil, \emph{\BIBforeignlanguage{English}{Nonlinear Systems}}.\hskip 1em plus 0.5em minus 0.4em\relax Upper Saddle River, NJ: Prentice Hall, 2002.

\bibitem{maurer2009empirical}
A.~Maurer and M.~Pontil, ``Empirical bernstein bounds and sample variance penalization,'' \emph{arXiv preprint arXiv:0907.3740}, 2009.

\end{thebibliography}

\end{document}